\documentclass[twoside]{article}
\usepackage[accepted]{aistats2020}
\usepackage{graphicx}
\usepackage{setspace}
\usepackage{caption}
\usepackage{subcaption}
\usepackage{amsmath}
\usepackage{listings}
\usepackage{amsthm}
\usepackage{indentfirst}
\usepackage[round]{natbib}
\usepackage{amssymb}
\usepackage{mathtools}
\usepackage{floatrow}
\usepackage{url}
\usepackage{hyperref}
\newtheorem*{rep@theorem}{\rep@title}
\newcommand{\newreptheorem}[2]{%
\newenvironment{rep#1}[1]{%
 \def\rep@title{#2 \ref{##1}}%
 \begin{rep@theorem}}%
 {\end{rep@theorem}}}
\makeatother

\newtheorem{theorem}{Theorem}
\newreptheorem{theorem}{Theorem}

\newtheorem{lemma}[theorem]{Lemma}
\newreptheorem{lemma}{Lemma}
\newtheorem{proposition}[theorem]{Proposition}
\newtheorem{definition}[theorem]{Definition}
\newtheorem{remark}[theorem]{Remark}

\newcommand\epssim{\stackrel{\mathclap{\epsilon, \alpha}}{\sim}}
\newcommand\cF{\mathcal{F}}
\newcommand\cG{\mathcal{G}}
\newcommand\cS{\mathcal{S}}
\newcommand\cC{\mathcal{C}}
\newcommand{\bR}{\mathbb{R}}

\DeclareMathOperator{\TV}{TV}
\newcommand{\DTV}[2]{\operatorname{TV}\left(#1, #2\right)}

\newcommand{\dist}{\ensuremath{\textnormal{\normalfont Dist}}}

\usepackage[dvipsnames]{xcolor}

\begin{document}

%

%

\twocolumn[

\aistatstitle{On the Sample Complexity of Learning Sum-Product Networks}


\aistatsauthor{ Ishaq Aden-Ali  \And Hassan Ashtiani }

\aistatsaddress{ McMaster University \\ adenali@mcmaster.ca \And McMaster University  \\zokaeiam@mcmaster.ca}
]
\begin{abstract}
Sum-Product Networks (SPNs)~\citep{Poon2011,darwiche2003} can be regarded as a form of deep graphical models that compactly represent deeply factored and mixed distributions. An SPN is a rooted directed acyclic graph (DAG) consisting of a set of leaves (corresponding to base distributions), a set of sum nodes (which represent mixtures of their children distributions) and a set of product nodes (representing the products of its children distributions). 

In this work, we initiate the study of the sample complexity of PAC-learning the set of distributions that correspond to SPNs. We show that the sample complexity of learning tree structured SPNs with the usual type of leaves (i.e., Gaussian or discrete) grows at most linearly (up to logarithmic factors) with the number of parameters of the SPN.

More specifically, we show that the class of distributions that corresponds to tree structured Gaussian SPNs with $k$ mixing weights and $e$ ($d$-dimensional Gaussian) leaves can be learned within Total Variation error $\epsilon$ using at most $\widetilde{O}(\frac{ed^2+k}{\epsilon^2})$ samples. A similar result holds for tree structured SPNs with discrete leaves. 

We obtain the upper bounds based on the recently proposed notion of distribution compression schemes~\citep{gaussian_mixture_tr}. More specifically, we show that if a (base) class of distributions $\cF$ admits an ``efficient'' compression, then the class of tree structured SPNs with leaves from $\cF$ also admits an efficient compression.

\end{abstract}

\section{Introduction}

A Sum-Product Network (SPN) \citep{Poon2011,darwiche2003} is a type of deep probabilistic model that can represent complex probability distributions. An SPN can be specified by its graphical model which takes the form of a rooted directed acyclic graph (DAG). The leaves of an SPN represent probability distributions from a fixed (simple and often parametric) class such as Bernoulli and Gaussian distributions. Higher-level nodes in the graph correspond to more complex distributions that are obtained by ``combining'' the lower-level distributions.
More specifically, each node of an SPN is either a leaf, a \emph{sum node}, or a \emph{product node}. Each sum/product node represents the mixture/product distribution of its children respectively. The use of sum and product ``operations'' allows for representation of increasingly more complex distributions, all the way to the root. The distribution that an SPN represents is the one that corresponds to its root node. In this work, our focus is on a powerful subclass of SPNs that take the form of rooted trees instead of rooted DAGs. We clarify that our results hold for tree structured SPNs and that for the remainder of this paper we will refer to tree structured SPNs simply as SPNs. See Figure \ref{spnfig} for an example of a simple SPN. 




SPNs can be considered a generalization of mixture models. The alternating use of sum and product operations in SPNs results in representation of highly structured and complex distributions in a concise form. The appealing property of SPNs in encoding deeply factored and mixed distributions becomes more evident when we increase their \emph{depth} -- allowing representation of more complex distributions~\citep{Delalleau,Martens2014}. This property is a major incentive for the use of SPNs in practice~\citep{Delalleau,Martens2014, Poon2011,Rashwan2018b, Zhao2016,Vergari2015,Adel2015}.


A fundamental open problem that we aim to address is characterizing the number of training instances needed to learn an SPN. More specifically, we want to establish the sample complexity of learning an SPN as a function of its depth and the number of its nodes (as well as the sample complexity of learning a single leaf).
In this work, we initiate the study of the sample complexity of SPNs within the standard \emph{distribution learning} framework (e.g.,~\cite{devroye_book}), where we are given an i.i.d. sample from an unknown distribution and we wish to find a distribution that---with high probability---is close to it in Total Variation (TV) distance. 

One important special case of SPNs are Gaussian Mixture Models (GMMs), which can be regarded as SPNs with only one sum node and a number of Gaussian leaves connected to it. Only recently, it has been shown that the number of samples required to learn GMMs is $\widetilde{\Theta}(p/\epsilon^2)$, where $p$ is the number of parameters of the mixture model~\citep{gaussian_mixture_tr}. It is therefore an intriguing question whether this result can be extended to SPNs. 

We will establish an upper bound on the sample complexity of learning tree structured SPNs, affirming that the sample complexity grows (almost) linearly with the number of parameters. As a concrete example, we show that the sample complexity of learning SPNs with fixed structure and Gaussian leaves is at most $\widetilde{O}(p/\epsilon^2)$ where $p$ is basically the number of parameters (the number of edges/weights in the graph plus the number of Gaussian parameters). Similar results also hold for SPNs with other usual types of leaves, including discrete (categorical) leaves.

We prove our results using the recently proposed notion of distribution compression schemes~\citep{gaussian_mixture_tr}. We obtain our sample complexity upper bounds by showing that if a class of distributions, $\cF$, admits a certain form of efficient sample compression, then the set of distributions that correspond to SPNs with leaves from $\cF$ is also efficiently compressible, as long as the number of edges in the SPN is bounded. A technical feature of this result is that the upper bound depends on the number of the edges, but has no extra dependence (e.g., no exponential dependence) on the depth of the SPN.



\subsection{Notation}
For a distribution $f$, the notation $S \sim f^{m}$ means $S$ is an i.i.d sample of size $m$ generated from $f$. We write the set $\{1, 2, \dots, N\}$ as $[N]$, and write $|B|$ to represent the cardinality of the set $B$. The empty set is defined as $\emptyset$ and $\log(\cdot)$ denotes logarithm in the natural base.

\section{The Distribution Learning Framework}

In this short section we formally define the distribution learning framework. A \emph{distribution learning method} is an algorithm that takes as input a sequence of i.i.d. samples generated from an unknown distribution $f$, then outputs (a description of) a distribution $\hat{f}$ as an estimate of $f$. We assume that $f$ is in some class of distributions $\cF$ (i.e., realizable setting) and we require $\hat{f}$ to be a member of this class as well (i.e., proper learning).
Let $f_1$ and $f_2$ be two probability distributions defined over $\bR^d$ and let $\mathcal{B}$ be the Borel sigma algebra over $\bR^d$. The TV distance is defined by
\begin{equation*}
\label{eq:TVdef}
\DTV{f_1}{f_2}
 \: \coloneqq \: \sup_{B \in \mathcal{B}} \int_B \big(f_1(x) - f_2(x)\big) \,\mathrm{d} x = 
 \frac{1}{2}\|f_1 - f_2\|_1 
\end{equation*}
where 
\(\|f\|_1 \coloneqq \int_{\bR^d} |f(x)|\,\mathrm{d} x\)
is the $L_1$ norm of $f$. We define what it means for two distributions to be $\epsilon$-close.

\begin{definition}[$\epsilon$-close]
A distribution $\hat{f}$ is {$\epsilon$-close} to $f$ if $\TV(f,\hat{f}) \leq \epsilon$.
\end{definition}

The following is a formal Probably Approximately Correct (PAC) learning definition for distribution learning with respect to $\cF$.

\begin{definition}[PAC-learning of distributions]\label{def:realizablelearning}
	A distribution learning method is called a PAC-learner for $\cF$ with sample complexity $m_{\cF}(\epsilon, \delta)$ if, for all distributions $f\in\mathcal F$ and all $\epsilon, \delta \in(0,1)$, given $\epsilon$, $\delta$, and an i.i.d.\ sample of size $m_{\cF}(\epsilon, \delta)$ from  $f$, with probability at least $1-\delta$ (over the samples) we have $TV(f, \hat{f}) \leq \epsilon$.
\end{definition}

\section{Main Results}

Here we state our main result regarding the sample complexity of learning SPNs with Gaussian leaves which are the most common forms of continuous SPNs.

\begin{theorem}[Informal]\label{thm:mainresult}
Let $\cS$ be any class of distributions that corresponds to SPNs with the same structure---having $k$ mixing weights and $e$ ($d$-dimensional Gaussian) leaves.
Then $\cS$ can be PAC-learned using at most $$\widetilde{O}\bigg(\frac{ed^2+k}{\epsilon^2}\bigg) $$ samples, where $\widetilde{O}()$ hides logarithmic dependencies on $\epsilon$, $e$, $d$, $k$ and $1/\delta$. 
\end{theorem}

The parameters of an SPN consist of the mixing weights of the sum nodes and the parameters of the leaves. The number of parameters of a $d$-dimensional Gaussian is $d^2$, so our upper bound is nearly linear in the total number of parameters of the SPN (i.e., $ed^2+k$).

One of the technical aspects of our upper bound is that it depends on the structure of SPN only through $e$ and $k$. In other words, the upper bound will be the same for learning deep vs. shallow structured SPNs as long as they have the same number of mixing weights and leaves. This result motivates the use of deeper SPNs from the information-theoretic point of view---especially given the fact that deeper SPNs can potentially encode a distribution much more efficiently~\citep{Delalleau,Martens2014}.


\begin{remark}[Tightness of the upper-bound] It is known~\citep{gaussian_mixture_tr} that the sample complexity of learning mixtures of $k$ $d$-dimensional Gaussians is at least $\widetilde{\Omega}(kd^2/\epsilon^2)$. Given the fact that GMMs are special cases of SPNs, we can conclude that our upper bound cannot be improved in general (i.e., it can nearly match the lower bound, e.g., for the special case of GMMs). However, it might still be possible to refine the bound by considering additional parameters, which is the subject of future research.

\end{remark}

The approach that we use in this paper is quite general and allows us to investigate SPNs with other types of leaves, including discrete leaves. 
In fact, studying SPNs with discrete leaves is a simpler problem than those with Gaussian leaves. Here we state the corresponding result for discrete SPNs.

\begin{theorem}[Informal]\label{thm:discretemainresult}
Let $\cS$ be any class of distributions that corresponds to SPNs with the same structure---having $k$ mixing weights and $e$ discrete leaves of support size $d$.
Then $\cS$ can be PAC-learned using at most $$\widetilde{O}\bigg(\frac{ed+k}{\epsilon^2}\bigg) $$ samples, where $\widetilde{O}()$ hides logarithmic dependencies on $\epsilon$, $e$, $d$, $k$, and $1/\delta$. 
\end{theorem}

\section{Sum Product Networks}
We begin this section by defining mixture and product distributions which are the fundamental building blocks of SPNs. As we are working with absolutely continuous probability measures we sometimes use distributions and their density functions interchangeably. Let 
\(
\Delta_k \coloneqq \{\: (w_1,\dots,w_k) \in \mathbb{R}^k \,:\, w_i\geq 0 ,\, \sum w_i=1 \:\}
\)
denote the $k$-dimensional simplex. 

\begin{definition}[Mixture distribution]
	Let $f_1,\dots,f_k$ be densities over domain $Z$. We call $f$ a $k$-mixture of $f_i$'s if it can be written in the following form 
	$$
	f ~\coloneqq~ \textstyle\sum_{i=1}^{k} w_{i}f_{i}
	$$
where $(w_1,\dots,w_k)\in \Delta_k$ are mixing weights. 
\end{definition}

\begin{definition}[Product distribution]
	Let $p_1,\dots,p_d$ be densities over domains $Z_1,\dots,Z_d$. Then a product density, $p$, over $\prod_{i=1}^{d} Z_i$ is defined by
$
p ~\coloneqq~ \prod_{i=1}^{d} p_i 
$
\end{definition}

\subsection{SPN Signatures}

In this subsection, we introduce the notion of \emph{SPN signatures} which help defining SPNs more formally. SPN signatures can be thought of as a recursive syntactic representation of SPNs. We find this syntactic representation useful in improving the clarity and preciseness of the statements that we will make in our main results and their proofs.

We now define base signatures which will later allow us to recursively define SPN signatures. Let $\cF$ denote a (base) class of distributions that are defined over $\mathbb{R}^d$. In the following, we define the set of ``base signatures'' which basically correspond to the leaves of SPNs. 

\begin{definition}[Base signatures]
The set of base signatures formed by $\cF$ over $\mathbb{R}^n$ is defined by the following set of tuples $$T_{\cF}^n \coloneqq \{(f,b): b \subset [n], |b|=d, f \in \cF\}$$
\end{definition}

The first element of each signature (tuple) is a symbol that represents a distribution in $\cF$. The second element of the tuple represents the subset of dimensions that the domain of $f$ is defined over. For example, the tuple $(f,\{1,2,5\})$ represents a distribution $f$ that is defined over the first, third and fifth dimensions of $\mathbb{R}^n$ (i.e., $d=3$ is dimensionality of the domain of $f$, and $n>d$ is the dimensionality of the domain of the whole SPN). The set $b$ is commonly referred to as the \emph{scope} of the distribution. 
An example of a set of base signatures are the base signatures formed by the class of $d$-dimensional Gaussians, $\cG$, given by $T_{\cG}^n = \{(g,b):b\subset [n], |b|=d, g \in \cG\}$. 

We are now ready to define SPN signatures. In fact, SPN signatures are ``generated'' recursively from the base signatures, by either taking the product or mixtures of the existing signatures.

\begin{definition}[SPN signatures]
Given a set of base signatures $T_{\cF}^n$, we (recursively) define the set of SPN signatures generated from $T_{\cF}^n$ -- denoted by $S(T_{\cF}^n)$ -- to consist of the following tuples:
\begin{enumerate}
    \item $T_{\cF}^n \in S(T_{\cF}^n)$
    \item If $\exists (z_{1},b_{1}), \dots ,(z_{k},b_{k}) \in S(T_{\cF}^n)$ and\\ $\forall i \neq j$ $b_{i} \cap b_{j} = \emptyset$, then \\ $\big(\left((z_{1},b_1)\times \dots \times (z_{k},b_k)\right),\bigcup\limits_{i=1}^{k} b_{i}\big) \in S(T_{\cF}^n)$
    \item If $\exists (z_{1},b), \dots ,(z_{k},b) \in S(T_{\cF}^n)$ and \\ $(w_{1}, 
    \dots, w_{k}) \in \Delta_{k}$, then \\ $\big(\left(w_{1}(z_{1},b_1) + \dots + w_{k}(z_{k},b_k)\right),b\big) \in S(T_{\cF}^n)$
\end{enumerate}
\end{definition}

The first rule of the above definition states that SPN signatures include the corresponding base signatures. The second rule defines new signatures based on the product of the existing ones. 
The resulting signature, $((z_{1}, b_1)\times \dots \times (z_{k}, b_k))$, is a string that is the concatenation of a number of substrings (i.e., each $z_i$ and $b_i$) and a number of symbols ($'\times'$, $'('$ and $')'$) in the given order. The second element of the tuple (i.e., $\cup b_i$) keeps track of the dimensions over which the signature is defined. Similarly, the third rule defines signatures based on the mixture of the existing signatures ($w_{i}$ are the string representation of a mixing weight).

One can take an SPN signature and create its corresponding visual (graph-based) representation based on the sum and product rules. See Figure~\ref{spnfig} for an example of an SPN and its corresponding signature. We will often switch between referring to an SPN as a distribution or as a rooted tree, and it will be clear what we are referring to from the context. 



\begin{figure}[h]
    \centering
    \includegraphics[width=0.95\linewidth]{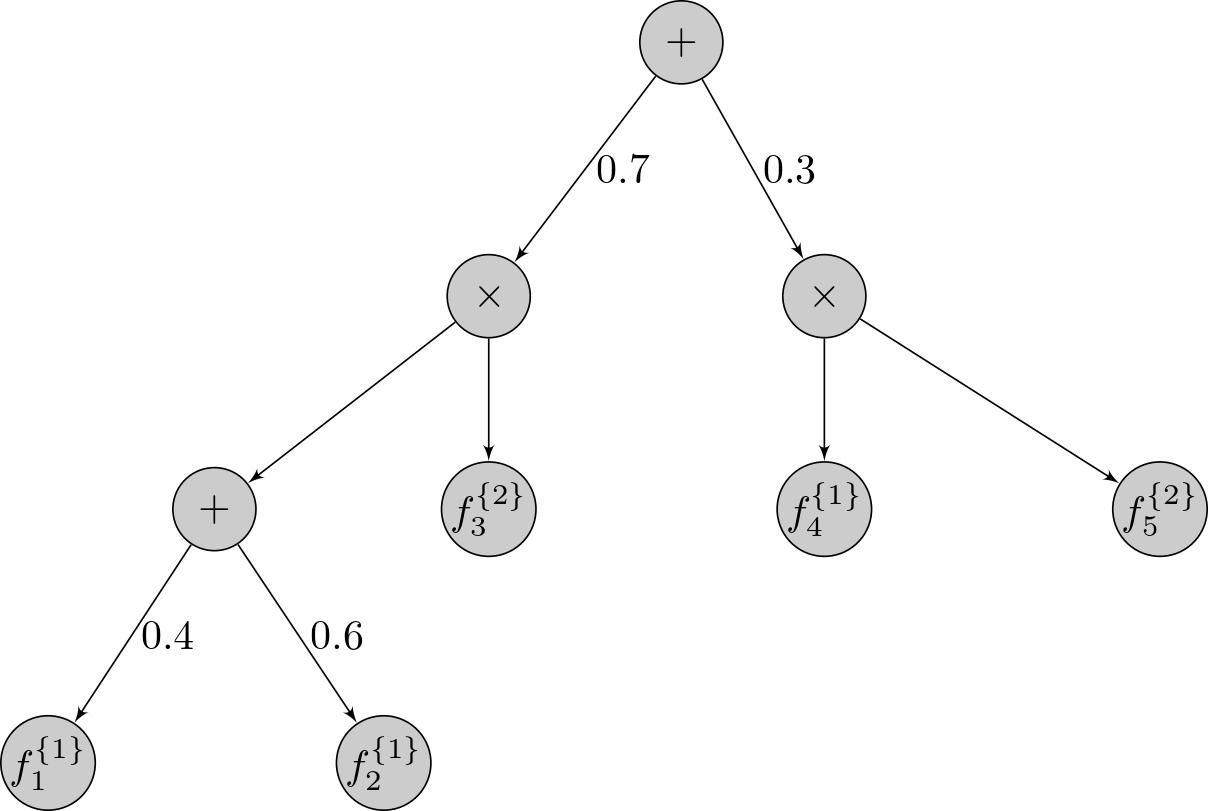}
    \caption{This figure depicts a simple SPN. Each leaf in this SPN represents a simple univariate distribution. To save space, the scope of each leaf is written in the superscript. Each + node represents a mixture distribution of its children, and its mixing weights are shown on the outgoing edges to its children. Each $\times$ node represents the product distribution of its children.
    The root node corresponds to the ultimate distribution that the SPN represents. The corresponding SPN signature for this SPN is the following: $\big( 
    \big(0.7 ( ( 0.4(f_1,\{1\})+0.6(f_2,\{1\}) ) \times (f_3,\{2\}) ) +\\  0.3 ( (f_4,\{1\}) \times (f_5,\{2\}) )\big)
    ,~\{1,2\}\big)$.}
    \label{spnfig}
\end{figure}

\subsection{ \texorpdfstring{$(\epsilon,\alpha)$}-Similarity and SPN Structures}

In this subsection we define $(\epsilon, \alpha)$-similarity between two signatures. This property is a useful tool that will allow us to simplify our proofs. $(\epsilon, \alpha)$-similarity can also be used to formally define what it means for two signatures to have the same \emph{structure}. 

Roughly speaking, two signatures are $(\epsilon,\alpha)$-similar if (i) their corresponding SPNs have the same structure, and (ii) all their corresponding weights are $\alpha$-similar, and (iii) all their corresponding leaves are $\epsilon$-close in TV distance. Here is the formal definition.

\begin{definition}[$(\epsilon,\alpha)$-similar]\label{similarity}
Given parameters $\epsilon \in [0,1]$ and $\alpha \in [0,1]$, we say two signatures $s=(z,b),\hat{s}=(\hat{z},\hat{b}) \in S(T_{\cF}^n)$ are $(\epsilon, \alpha)$-similar, $s \epssim \hat{s}$, if they satisfy one of the following three properties:
\begin{enumerate}
    \item
    \begin{itemize}
        \item $b=\hat{b}$ and
        \item $z,\hat{z} \in \cF$ and
        \item $\TV(z, \hat{z}) \leq \epsilon$
    \end{itemize}
    \item $\exists z_{1},\dots,z_{k},\hat{z}_{1},\dots,\hat{z}_{k}$ and $b_{1},\dots,b_{k},\hat{b}_{1},\dots,\hat{b}_{k} $ such that
    \begin{itemize}
        \item  $z=((z_{1},b_1) \times \dots \times (z_{k},b_k))$ and  
        \item $\hat{z}=((\hat{z}_{1},b1) \times \dots \times (\hat{z}_{k}, b_k))$ and
        \item $b = \hat{b} = \bigcup\limits_{i=1}^{k}b_{i}$ and
        \item $\forall i \enskip (z_{i},b_{i}) \epssim (\hat{z}_{i},\hat{b}_{i})$
    \end{itemize}
    \item $\exists z_{1},\dots,z_{k},\hat{z}_{1},\dots,\hat{z}_{k}$ and\\ $(w_{1},\dots,w_{k}),(\hat{w}_{1},\dots,\hat{w}_{k}) \in \Delta_{k}$ such that
    \begin{itemize}
        \item  $z=(w_{1}(z_{1}, b_1) + \dots + w_{k}(z_{k}, b_k))$ and 
        \item $\hat{z}=(\hat{w}_{1}(\hat{z}_{1}, b_1) + \dots + \hat{w}_{k}(\hat{z}_{k}, b_k))$ and
        \item $\forall i \enskip (z_{i},b) \epssim (\hat{z}_{i},\hat{b})$ and
        \item $\forall i \enskip \vert \hat {w}_{i} - w_{i} \vert \leq \alpha$
    \end{itemize}
\end{enumerate}
\end{definition}


We slightly abuse notation by writing $z,\hat{z} \in \cF$ and $TV(\hat{z}, z) \leq \epsilon$, since the $z$ and $\hat{z}$ are strings/symbols; yet, here we just mean the distribution in $\cF$ that corresponds to that symbol. $(\epsilon,\alpha)$-similarity is a useful property that we will later use in our proofs. 

More immediately, we need $(\epsilon,\alpha)$-similarity to formally define what it means for two signatures to have the \emph{same structure}.

\begin{definition}[same structure]\label{def:samestruct}
We say two signatures $s,\hat{s} \in S(T_{\cF}^n)$ have the same structure if they are $(1, 1)$-similar and we denote this by $s \equiv \hat{s}$.

\end{definition}

Note that the TV distance between two distributions is at most $1$ and the absolute difference between two weights is at most $1$. Therefore, the fact that two signatures are $(1,1)$-similar just means that their corresponding SPNs have the same structure, in the sense that the way their nodes are connected is the same, and the scope of corresponding sub trees is the same (with no actual guarantee on the ``closeness'' of their leaves or weights).


Note that $s \equiv \hat{s}$ is an equivalence relation, therefore we can talk about the equivalence classes of this relation. We use these equivalence classes to give the following definition.

\begin{definition}[An SPN structure]\label{def:SPNstructure}
Given a signature $s \in S(T_{\cF}^n)$, we define an SPN structure, $[s]_{\cF}$, as the equivalence class of $s$ $$[s]_{\cF} \coloneqq \left\{s' \in S(T_{\cF}^n) : s \equiv s' \right\}$$
\end{definition}

Essentially, an SPN structure is a set of signatures that represent distributions that all have the same structure. This is a useful definition as it will allow us to precisely state our results. Finally, we denote by $\dist([s]_{\cF})$ the set of all distributions that correspond to the signatures in the equivalence class $[s]_{\cF}$. We are now ready to state our main results formally.

\subsection{Main Results: Formal}

The proof of these results can be found in Section~\ref{sec::proofs}.

\begin{reptheorem}{thm:mainresult}
Let $\cG$ be the class of $d$-dimensional Gaussians. For every SPN structure $[s]_{\cG}$, the class $\dist\left([s]_{\cG}\right)$ can be learned using $$\widetilde{O}\bigg(\frac{ed^2+k}{\epsilon^2}\bigg) $$ samples, where $e$ and $k$ are the number of leaves and the number of mixing weights of (every) $f \in \dist([s]_{\cF})$ respectively.
\end{reptheorem}

\begin{reptheorem}{thm:discretemainresult}
Let $\cC$ be the class of discrete distributions with support size $d$. For every SPN structure $[s]_{\cC}$, the class $\dist\left([s]_{\cC}\right)$ can be learned using $$\widetilde{O}\bigg(\frac{ed+k}{\epsilon^2}\bigg) $$ samples, where $e$ and $k$ are the number of leaves and the number of mixing weights of (every) $f \in \dist([s]_{\cF})$ respectively.
\end{reptheorem}


\section{Distribution Compression Schemes}

In this section we provide an overview of distribution compression schemes and their relation to PAC-learning of distributions. Distribution compression schemes were recently introduced in~\citep{gaussian_mixture_tr} as a tool to study the sample complexity of learning a class of distributions. Here, the high-level use-case of this approach is that if we show a class of distributions admits a certain notion of compression, then we can bound the sample complexity of PAC-learning with respect to that class of distributions.

Let us fix a class of distributions, $\cF$. A distribution compression scheme for $\cF$ consists of an \emph{encoder} and a \emph{decoder}. The encoder, knowing the true data distribution $f\in \cF$, receives an i.i.d. sample $S\sim f^m$ of size $m$, and tries to ``encode'' $f$ using a small subset\footnote{Technically, the encoder can use the same instance multiple times in the message, so we have a sequence rather than a set.} of $S$ and a few extra bits. On the other hand, the \emph{decoder}, unaware of $f$, aims to reconstruct (an approximation of) $f$ using the given subset of samples and the bits. Roughly speaking, $\cF$ is compressible if there exist a decoder and an encoder such that for any $f\in \cF$, the decoder can recover (a good approximation of) $f$ based on the given information from the encoder.

More precisely, suppose that the encoder always uses ``short'' messages to encode any $f\in\cF$: it uses a sequence of at most $\tau$ instances from $S$ and at most $t$ extra bits. Also, suppose that for all $f\in \cF$, the decoder receives the encoder's message and with high probability outputs an $\hat{f} \in \cF$ such that $TV(f, \hat{f})\leq \epsilon$. In this case, we say that $\cF$ admits $(\tau, t, m)$ compression, where $\tau$, $t$, and $m$ can be functions of the accuracy parameter, $\epsilon$.

The difference between this type of sample compression and the more usual notions of compression is that we not only use bits, but also use the samples themselves to encode a distribution. This extra flexibility is essential---e.g., the class of univariate Gaussian distributions (with unbounded mean) has infinite metric entropy and can be compressed only if on top of the bits we use samples to encode the distribution.

\subsection{Formal Definition of Compression Schemes}
In this section, we provide a formal definition of distribution compression schemes.
\begin{definition}[decoder~\citep{gaussian_mixture_tr}]
A \emph{decoder~\citep{gaussian_mixture_tr}} for $\cF$ is a deterministic function 
$\mathcal{J}:\bigcup_{n=0}^{\infty} Z^n \times \bigcup_{n=0}^{\infty} \{0,1\}^n 
\rightarrow \cF$, which takes a finite sequence of elements of $Z$ and a finite sequence of bits, and outputs a member of $\cF$. 
\end{definition}

\begin{definition}[compression schemes]
\label{def_compression}
Let $\tau(\epsilon),t(\epsilon),m(\epsilon):(0,1)\rightarrow \mathbb{Z}_{\geq0}$ be functions.
We say $\cF$ admits $\left(\tau(\epsilon),t(\epsilon),m(\epsilon)\right)$  compression if there exists a decoder $\mathcal{J}$ for $\cF$ such that for any distribution $f \in \cF$, the following holds:
\begin{itemize}
\item For any $\epsilon \in (0,1)$, if a sample $S$ is drawn from $f^{m(\epsilon)}$, then with probability at least $2/3$, there exists a sequence $L$ of at most $\tau(\epsilon)$ elements of $S$, and a sequence $B$ of at most $t(\epsilon)$ bits, such that $\DTV{f}{\mathcal{J}(L,B)}\leq \epsilon$.
\end{itemize}
\end{definition}

Briefly put, the definition states that with probability $2/3$,
there is a (short) sequence $L$ of elements from $S$
and a (short) sequence $B$ of additional bits,
from which $f$ can be approximately reconstructed. This probability can be increased to $1-\delta$ by generating a sample of size $m(\epsilon)\log(1/\delta)$. The following technical lemma gives an efficient compression scheme for the class of $d$-dimensional Gaussians.

\begin{lemma}[Lemma 4.2 in \citep{gaussian_mixture_tr}]
\label{lem:gausscompress}
For any positive integer $d$, the class $\cG$ of $d$-dimensional Gaussians admits an
$$
\big(~ O(d\log (2d)) ,\, O(d^2 \log (2d) \log(d/\epsilon)) ,\, O(d \log (2d)) ~\big)
$$
compression scheme.
\end{lemma}

\subsection{From Compression to Learning}
The following theorem draws a connection between compressibility and PAC-learnability. It states that the sample complexity of PAC-learning a class of distributions can be upper bounded if the class admits a distribution compression scheme.


\begin{theorem}[Compression implies learning, Theorem 3.5 in ~\citep{gaussian_mixture_tr}]
\label{thm:compressionimplieslearning}
Suppose $\cF$ admits $(\tau(\epsilon),t(\epsilon),m(\epsilon))$  compression. Then $\cF$ can be PAC-learned using 
\begin{align*}
 \widetilde{O}
\left(
m\Big(\frac \epsilon 6\Big)  + \frac{t(\epsilon/6)+\tau(\epsilon/6)}{\epsilon^2} 
\right) \textnormal{ \emph{samples.}}
\end{align*}
\end{theorem}

The idea behind the proof of this theorem is simple: if a class admits a compression scheme, then the learner can try to simulate all the messages that the encoder could have possibly sent, and use the decoder on them to find the corresponding outputs. The problem then reduces to learning from a finite class of candidate distributions (see~\citep{gaussian_mixture_tr} for details).

\section{Compressing SPNs}

In this section we show (roughly) that if a class of distributions, $\cF$, is compressible, then a class of SPNs with fixed structure and leaves from $\cF$ is also compressible.  We use this result as a crucial step in proving our main results. We give a full proof of the following Theorem in the supplement.

\begin{theorem}\label{thm:strongspncompression}
Let $\cF$ be a class that admits $(\tau(\epsilon),t(\epsilon),m(\epsilon))$ compression. For every SPN structure $[s]_{\cF}$, the class \dist$([s]_{\cF})$ admits $$\big(e\tau(\epsilon/3n),et(\epsilon/3n)+k\log_{2}{(3k/2\epsilon)},48m(\epsilon/3n)e\log(6e)/\epsilon\big)$$  compression where $k$, $e$ and $n$ are the number of weights, the number of leaves, and the dimension of the domain of the distributions in \dist$([s]_{\cF})$, respectively.
\end{theorem}

\subsection{Overview of Our Techniques}
In this subsection, we give a high level overview of our technique. As was stated in Theorem~\ref{thm:strongspncompression}, given an SPN structure $[s]_{\cF}$, we want to derive a compression scheme for the class of distributions $\dist([s]_{\cF})$ as long as the class $\cF$ is compressible. Our compression scheme utilizes an encoder and decoder for the class $\cF$. Our encoder can encode any $f \in \dist([s]_{\cF})$ in the following way: given $m$ samples from $f$, for each leaf $f_{i}$ the encoder can likely choose a sequence of $\tau$ samples (from the $m$ samples of $f$) and $t$ bits such that a decoder for the class $\cF$ can outputs an $\hat{f}_{i} \in \cF$ that is an accurate approximation of $f_i$. Furthermore, for each sum node $j$, we discretize its mixing weights $(w_{1},\dots,w_{k_{j}})$ with high accuracy. Our discretization, $(\hat{w}_{1},\dots,\hat{w}_{k_{j}})$, can be encoded \emph{exactly} using some bits.

Our decoder for the class $\dist([s]_{\cF})$ decodes the message from the encoder in the following way: Our decoder is given the discretize mixing weights for each sum node directly in the form of bits, so nothing more needs to be done for the weights. The decoder is also given $\tau$ samples and $t$ bits for each leaf $f_{i}$. We can use the decoder for  the class $\cF$ to reconstruct $\hat{f}_{i} \in \cF$ that is an accurate approximation for $f_{i}$, with high probability. Our decoder thus outputs the reconstructed SPN $\hat{f} \in \dist([s]_{\cF})$ with leaves $\hat{f}_{i}$ and discretized mixing weights $(\hat{w}_{1},\dots,\hat{w}_{k_{j}})$ for each sum node $j$. Finally, we show that the decoders reconstruction, $\hat{f}$, is $\epsilon$-close to $f$ with high probability.

\subsection{Results}

In this subsection, we work towards proving a less general version of Theorem~\ref{thm:strongspncompression} that has a simpler and more intuitive proof. With this in mind, we introduce a few definitions. 

\begin{definition}[$\epsilon$-net]
Let $\epsilon\geq 0$. We say $N\subseteq X$ is an $\epsilon$-net for $X$ in metric $d$ if for each $x\in X$ there exists some $y\in N$ such that $d(x,y)\leq \epsilon$.
\end{definition}

\begin{definition}[path weight]
Let $e$ be the number of leaves in an SPN. For any index $i\in[e]$, the path weight, $W_{i}$, of the $i$th leaf of an SPN is the product of all the mixing weights along the unique path from the root to the $i$th leaf.
\end{definition}

In other words, when sampling from an SPN distribution, the path weight of a leaf is the probability of getting a sample from that leaf. 
We say a leaf in an SPN is \emph{negligible} if its path weight is less than $\epsilon/3e$, where $e$ is the number of leaves in the SPN. When we say a class of SPNs has no negligible leaves, we mean none of the distributions in the class has negligible path weights. We now proceed to prove the following lemma.

\begin{lemma}\label{lem:weakspncompression}
Let $\cF$ be a class that admits $(\tau(\epsilon),t(\epsilon),m(\epsilon))$ compression. For every SPN structure $[s]_{\cF}$, the class \dist$([s]_{\cF})$ (with no negligible leaves), admits $$\big(e\tau(\epsilon/2n),et(\epsilon/2n)+k\log_{2}{(k/\epsilon)},48m(\epsilon/2n)e\log(6e)/\epsilon\big)$$  compression where $k$, $e$ and $n$ are the number of weights, the number of leaves, and the dimensionality of the domain of the distributions in \dist$([s]_{\cF})$, respectively.
\end{lemma}

To prove the above lemma, we first need to show that if two SPN signatures are $(\epsilon,\alpha)$-equivalent, then there is a direct relationship between the SPNs they represent. The proof can be found in the supplement.
\begin{lemma}\label{lem:epsalphesim}
Given that two signatures $s,\hat{s} \in[s]_{\cF}$ are $(\epsilon,\alpha)$-similar, their corresponding SPNs $f,\hat{f} \in \dist([s]_{\cF})$ satisfy
\begin{align*}\label{err_tree}
\TV(\hat{f},f) \leq n\epsilon + k\alpha/2
\end{align*}
where $k$ and $n$ are the number of weights in and the dimensionality of the domain of both $f$ and $\hat{f}$ respectively.
\end{lemma}

We can use the above result to prove lemma \ref{lem:weakspncompression}. 

\begin{proof}[Proof of Lemma \ref{lem:weakspncompression}]

We want to show that given $48m(\epsilon/2n)e\log(6e)/\epsilon$ samples from $f$, we can construct $\hat{f}$ that is $\epsilon$-close to $f$ with probability $2/3$. For any index $i\in[e]$, the $i$-th leaf of $f$ is given by $f_{i} \in \cF$.

\textit{Encoding:} We define $m_{0} \coloneqq 48m(\epsilon/2n)e\log(6e)/\epsilon$ to be the number of samples we have from $f\in\dist([s]_{\cF})$. Since we have $m_{0}$ samples (and none of the path weights are negligible) using a standard Chernoff bound together with a union bound, there are no less than $m(\epsilon/2n)\log6e$ samples for every leaf of $f$, with probability at least $5/6$. Given this many samples for each leaf, there exists a sequence of $\tau(\epsilon/2n)$ samples and $t(\epsilon/2n)$ bits such that a decoder for the class $\cF$ outputs $\hat{f}_{i} \in \cF$ that satisfies

\begin{equation}\label{eq:err_leaves}
\TV(f_{i},\hat{f}_{i}) \leq \frac{\epsilon}{2n}
\end{equation}

with probability no less than $1-1/6e$. Finally, using a union bound, the failure\footnote{Failure here is either our leaves not getting $m(\epsilon/2n)\log6e$ samples or not having a sequence of $\tau(\epsilon/2n)$ samples and $t(\epsilon/2n)$ bits such that a decoder can output a good approximation for each leaf.} probability of our encoding is no more than $1/3$.

Let $l \in \mathbb{N}_{+}$ be the number of sum nodes. Let $j\in[l]$ be an index; for the $j$-th sum node, we can construct an $(\epsilon/k)$-net in $\ell_{\infty}$ of size $(k/\epsilon)^{k_j}$ for its mixing weights, where $k_j$ is the number of mixing weights of the $j$th sum node. There exists, in each sum node's net, an element $(\hat{w}_{1}, \dots , \hat{w}_{k_j}) \in \Delta_{k_{j}}$ such that 
\begin{align}\label{eq:err_weights}
\Vert (\hat{w}_{1}, \dots , \hat{w}_{k_j}) - (w_{1}, \dots , w_{k_j}) \Vert_{\infty} \leq \frac{\epsilon}{k}
\end{align}
For each sum node $j$, we can encode $(\hat{w}_{1}, \dots , \hat{w}_{k_j})$ using no more than $k_{j}\log_{2}(k/\epsilon)$ bits; in total we need no more than $k\log_{2}(k/\epsilon)$ bits to encode all the weights in the SPN. Taking everything into account, we have $e\tau(\epsilon/2n)$ instances and $et(\epsilon/2n)+k\log_{2}{(k/\epsilon)}$ total bits that we use to encode the SPN $f$. 

\textit{Decoding:} the decoder directly receives $k\log_{2}{(k/\epsilon)}$ bits that correspond to $(\hat{w}_{1}, \dots , \hat{w}_{k_j})$, for each sum node $j$. We also receive, for each leaf $f_{i}$, $\tau(\epsilon/2n)$ instances and $t(\epsilon/2n)$ bits such that the decoder for the class $\cF$ outputs $\hat{f}_{i} \in \cF$ that satisfies Equation~(\ref{eq:err_leaves}). Our decoder will thus output an SPN $\hat{f} \in \dist([s]_{\cF})$ where the $i$-th leaf is $\hat{f}_{i}$ and each sum node $j$ has mixing weights $(\hat{w}_{1}, \dots , \hat{w}_{k_j})$.

To complete the proof, we need to show that our reconstruction, $\hat{f}$ is $\epsilon$-close to $f$ with probability $2/3$. Our encoding succeeds with probability no less than $2/3$ , so we simply need to show that $\TV(f,\hat{f})\leq \epsilon$. Let $s$ and $\hat{s}$ represent the SPN signatures of $f$ and $\hat{f}$ respectively. By equations (\ref{eq:err_leaves}) and (\ref{eq:err_weights}) and Definition~\ref{similarity}, we have that $s$ and $\hat{s}$ are $(\epsilon/2n,\epsilon/k)$-similar. Using lemma \ref{lem:epsalphesim} we have that $\TV(f,\hat{f})\leq \epsilon$. This completes our proof.
\end{proof}

\section{Additional Proofs}\label{sec::proofs}
In this section we prove some of the results stated in earlier sections.

\subsection{Proof of Theorem \ref{thm:mainresult}}

\begin{proof}[Proof of Theorem \ref{thm:mainresult}]
Let $\cG$ be the class of $d$-dimensional Gaussians. Combining Theorem \ref{thm:strongspncompression} and  lemma \ref{lem:gausscompress}, we have the following: for any SPN structure $[s]_{\cG}$ the class $\dist([s]_{\cG})$ -- where each distribution in $\dist([s]_{\cG})$  has $e$ leaves and has domain with dimensionality $n$ -- admits a
\begin{align*}
\bigg(&O\big(ed\log(2d)\big),\\
&O\big(ed^2\log(2d)\log(3dn/\epsilon)\big)+k\log_{2}{(3k/\epsilon)},\\
&O\big(d\log(2d)e\log(6e)/\epsilon \big) \bigg)
\end{align*}
compression scheme.
Using Theorem \ref{thm:compressionimplieslearning} shows that this class can be learned using $\widetilde{O}((ed^2+k)/\epsilon^2)$ samples, which completes the proof.
\end{proof}

\subsection{Proof of Theorem \ref{thm:discretemainresult}}

To prove Theorem~\ref{thm:discretemainresult} we need the following lemma.

\begin{lemma}\label{lem:compressioncategory}
The class of discrete distributions, $\cC$, with support size $d$ admits $(0,d\log(d/\epsilon),0)$ compression.
\end{lemma}

\begin{proof}
The proof is quite simple. We only need to compress the $d$ parameters $(p_1,\dots,p_d)$ directly. We cast an $(\epsilon/d)$-net in $l_{\infty}$ of size $(\epsilon/d)^d$ for the $d$ parameters. There exists an element in our $(\epsilon/d)$-net, $(\hat{p}_1,\dots,\hat{p}_d)$, that satisfies
\begin{align*}
\Vert (\hat{p}_{1}, \dots , \hat{p}_{d}) - (p_{1}, \dots , p_{d}) \Vert_{\infty} \leq \frac{\epsilon}{d}
\end{align*}
So for any $i \in [d]$, we have $\vert \hat{p}_{i} - p_{i} \vert \leq \epsilon/d$. We can thus encode $(\hat{p}_1,\dots,\hat{p}_d)$ using no more than  $d\log_{2}(\epsilon/d)$ bits. The decoder receives these discretized weights (in bits) and directly outputs $\hat{f} = (\hat{p}_1,\dots,\hat{p}_d)$ which is $\epsilon$-close to $f$. This completes the proof.
\end{proof}

\begin{proof}[Proof of Theorem \ref{thm:discretemainresult}]
Let $\cC$ be the class of discrete distributions with support size $d$. Combining Theorem \ref{thm:strongspncompression} and  lemma \ref{lem:compressioncategory} we have the following:
for any SPN structure $[s]_{\cC}$ the class $\dist([s]_{\cC})$ -- where each distribution in $\dist([s]_{\cC})$  has $e$ leaves and has domain with dimensionality $n$ -- admits a
$$(0,ed\log_{2}(3dn/\epsilon)+k\log_{2}{(3k/\epsilon)},0)$$ compression scheme.
Using Theorem \ref{thm:compressionimplieslearning} we have that this class can be learned using $\widetilde{O}((ed+k)/\epsilon^2)$ samples, which completes the proof.
\end{proof}

\section{Previous Work}

Distribution learning is a broad topic of study and has been investigated by many scientific communities (e.g., ~\citep{kearns,devroye_density_estimation_first,silverman}). There are many measures of distance to choose from when one wishes to measure the similarity of two distributions. In this work we use the TV distance which has been applied to derive numerous bounds on the sample complexity of learning GMMs \citep{gaussian_mixture_tr, ashtiani2017sample, onedimensional, DK14, suresh2014near} as well as other types of distributions (see~\cite{Diakonikolas2016,chan2013learning,diakonikolas2016efficient} and references therein). There are many other common measures of similarity used for density estimations such as Kullback-Leibler (KL) divergence and general $L_p$ distances with $p>1$. Unfortunately, for the simpler problem of learning GMMs it can be shown that the sample complexity of learning with respect to KL divergence and general $L_p$ distances must depend on structural properties of the distribution while the same does not hold for the TV distance. For more details on this see \citep{gaussian_mixture_tr}.

Research on SPNs has primarily been focused on developing practical methods that learn appropriate structures for SPNs \citep{Dennis2012,Gens2013,Peharz2013,lee2013online,Dennis2015,Vergari2015,Rahman2016,Trapp2016,hsu2017online,Dennis2017,jaini2018prometheus,Rashwan2018b,Bueff2018,trapp2019bayesian} as well as methods that learn the parameters of SPNs \citep{Poon2011,gens2012discriminative,peharz2014learning,desana2016learning,Rashwan2016,Zhao2016,jaini2016online,Trapp2018,Rashwan2018a,peharz2019}. This line of research is not directly related to our work since they are interested in the development of efficient algorithms to be used in practice.  

There has also been some theoretical work showing that increasing the depth of SPNs provably increases the representational power of SPNs \citep{Martens2014,Delalleau} as well as some work showing the relationship between SPNs and other types of probabilistic models \citep{jaini2018deep}. Although these papers investigate theoretical properties of SPNs, they are not directly related to our work as we are interested in the question of sample complexity.


\section{Discussion}

Loosely put, in this work we have derived upper bounds on the sample complexity of learning the class of tree structured SPNs with Gaussian leaves and the class of tree structured SPNs with discrete leaves. Our results hold for SPNs that are in the form of rooted trees. Therefore, it would be interesting to characterize the sample complexity of learning general SPNs (in the form of rooted DAGs). Moreover, our upper bounds hold for learning in the realizable setting, thus another interesting open direction is to extend our results to the agnostic setting. Although we can provide a simple lower bound on the sample complexity of learning SPNs -- based on the fact that mixture models can be viewed as a special case of SPNs -- an interesting open question is whether we can determine a tighter lower bound for SPNs of arbitrary depth. We leave these directions for future work.

\section{Acknowledgements}
This research was supported by an NSERC Discovery Grant. 

\newpage

\bibliographystyle{apalike}
\bibliography{main}

\newpage

\onecolumn

\appendix
\section{Omitted Proofs}
\subsection{Proof of Lemma 21}

To prove lemma 21, we need the following proposition. The following proposition is standard and can be proved, e.g., using the coupling characterization of the TV distance.

\begin{proposition} \label{prop:TVprod} For $i\in [d]$, let $p_i$ and $q_i$ be probability distributions over the same domain $Z$. Then
	$\|\Pi_{i=1}^{d} p_i - \Pi_{i=1}^{d} q_i \|_{1} \leq \sum_{i=1}^{d} \|p_i - q_i\|_{1}$.
\end{proposition}

\begin{proof}[Proof of Lemma 21]

Two signatures that are $(\epsilon,\alpha)$-similar have corresponding SPNs that have the same structure. Since any SPN can be considered the composition of smaller trees (that are also SPNs), we prove this by structural induction. Note that $\TV(f,\hat{f}) \leq n\epsilon + k\alpha/2$ implies $\|f-\hat{f}\|_{1} \leq 2n\epsilon + k\alpha$ by definition.

\textbf{Base Case:}
Trees of height $0$ are simply leaves. By the definition of $(\epsilon,\alpha)$-similarity, two leaves are $\epsilon$-close. Thus, the inductive hypothesis is satisfied.

\textbf{Inductive step:}

Suppose that the root nodes of both $f$ and $\hat{f}$ have $m$ children. For $i\in[m]$, let $f_{i}$ and $\hat{f}_{i}$ be the $i$th children (that are smaller SPNs) of the root nodes of the SPNs $f$ and $\hat{f}$ respectively.  Assume $\forall i \in [m]$ the inductive hypothesis holds between $f_{i}$ and $\hat{f}_{i}$. There are two possibilities for SPNs $f$ and $\hat{f}$: $1)$ They are rooted with a product node or $2)$ they are rooted with a sum node. 

\emph{Case 1:} The root node is a product node with $m$ children. This gives us
\begin{align*}
    \Vert \hat{f} - f \Vert_{1} &= \Big\Vert \prod_{i = 1}^{m}{\hat{f}_{i}} - \prod_{i = 1}^{m}{f_{i}} \Big\Vert_{1} \leq \sum_{i=1}^{m}\Vert \hat{f}_{i} - f_{i} \Vert_{1} \\
    &\leq \sum_{i=1}^{m}( 2n_{i}\epsilon + k_{i}\alpha)= 2n\epsilon + k\alpha
\end{align*}
Where $k_i$ and $n_i$  are the number of weights in and the dimensionality of the domain of $f_{i}$ respectively. The first inequality follows from Proposition \ref{prop:TVprod}. 

\emph{Case 2:} The root node is a sum node with $m$ children. This give us
\begin{align*}
    \Vert \hat{f} - f \Vert_{1} &= \Big\Vert\sum_{i=1}^{m}(\hat{w}_{i}\hat{f_{i}}-w_{i}f_{i})\Big\Vert_{1}\\
    &\leq \Big\Vert\sum_{i=1}^{m} w_{i}  (\hat{f}_{i} - f_{i})\Big\Vert_{1}  + \Big\Vert\sum_{i =1}^{m}( \hat{w}_{i} - w_{i}) \hat{f}_{i} \Big\Vert_{1}\\
    &\leq \sum_{i=1}^{m} w_{i} \Vert \hat{f}_{i} - f_{i} \Vert_{1} +  \sum_{i =1}^{m} \vert \hat{w}_{i} - w_{i} \vert \Vert \hat{f}_{i} \Vert_{1}\\
    &\leq \sum_{i=1}^{m} w_{i} \cdot (2n\epsilon + k_{i}\alpha) + \sum_{i=1}^{m} \alpha \cdot 1\\
    &\leq 2n\epsilon \sum_{i=1}^{m} w_{i} +  \bigg(\sum_{i=1}^{m} w_{i} \cdot \sum_{i=1}^{m} k_{i}\alpha \bigg)+ \sum_{i=1}^{m} \alpha \cdot 1\\
    &= 2n\epsilon \cdot 1 + 1 \cdot \sum_{i=1}^{m} k_{i}\alpha+ \sum_{i=1}^{m} \alpha \cdot 1\\
    &= 2n\epsilon + k\alpha
\end{align*}
Where the first two inequalities hold from the properties of a norm. This completes our proof.
\end{proof}

\newpage
\subsection{Proof of Theorem 17}

The proof of Theorem 17 mirrors many aspects of the proof of Lemma 21.

\begin{proof}[Proof of Theorem 17]
We want to show that given $48m(\epsilon/3n)e\log(6e)/\epsilon$ samples from $f$, we can construct $\hat{f}$ that is $\epsilon$-close to $f$ with probability $2/3$.

\textit{Encoding:} We encode the mixing weights of the sum nodes the exact same way as in the proof for Lemma 21, but we use an $(2\epsilon/3k)$-net instead of an $(\epsilon/k)$-net.

Let $i\in[e]$ be an index. Recall that a leaf $f_{i}$ is negligible if its path weight $W_i$ is less than $\epsilon/3e$. We can split leaves into two groups: leaves that are negligible and leaves that are non-negligible. For non-negligible leaves, using a standard Chernoff bound together with a union bound, we have $m(\epsilon/3n)\log(6e)$ samples from all non-negligible leaves with probability no less than $5/6$. Given this many samples for each non-negligible leaf, there exists a sequence of $\tau(\epsilon/3n)$ samples and $t(\epsilon/3n)$ bits such that a decoder for the class $\cF$ outputs $\hat{f}_{i}\in \cF$ that satisfies
\begin{align}\label{eq:strongerr_leaves}
\TV(f_{i},\hat{f}_{i}) \leq \frac{\epsilon}{3n}
\end{align}

with probability no less than $1-1/6e$. Using a union bound, Equation~\ref{eq:strongerr_leaves} holds for all non-negligible leaves simultaneously with probability no less than $5/6$.

We \emph{cannot} make any guarantees on the number of samples that come from negligible leaves. For each negligible leaf $f_{i}$, we pick an arbitrary sequence of $\tau(\epsilon/3n)$ samples and $t(\epsilon/3n)$ bits. The decoder for the class $\cF$ will thus output an arbitrary $\hat{f_{i}} \in \cF$ such that $\TV(f_{i},\hat{f}_{i}) \leq 1$. As we will show shortly, we do not need to do any better for negligible leaves. Finally, using a union bound, the failure \footnote{Failure  here  is  either  our  non-negligible leaves not  getting $m(\epsilon/3n)\log(6e)$ samples or not having a sequence of $\tau(\epsilon/3n)$ samples and $t(\epsilon/3n)$ bits such that a decoder can output a good approximation for each non-negligible leaf.} probability of our encoding is no more than $1/3$.

\textit{Decoding:} The mixing weights for our sum nodes are decoded in the same way as as in the proof for Lemma 21. Our decoder receives a sequence of $\tau(\epsilon/3n)$ and $t(\epsilon/3n)$ bits for each leaf $f_{i}$. For non-negligible leaves, the decoder for the class $\cF$ will output $\hat{f}_{i}\in \cF$ that satisfies Equation (\ref{eq:strongerr_leaves}). For non-negligible leaves, the decoder outputs also outputs an $\hat{f}_{i}\in \cF$, but we have no guarantee of how accurately it approximates $f_{i}$. Our decoder thus outputs an SPN $\hat{f} \in \dist([s]_{\cF})$ where the $i$-th leaf is $\hat{f}_{i}$ and each sum node $j$ has mixing weights $(\hat{w}_{1}, \dots , \hat{w}_{k_j})$.

We need to show that $\TV(\hat{f},f) \leq \epsilon$ with probability $2/3$. Our encoding succeeds with probability $2/3$, so we only need to show $\TV(\hat{f},f) \leq \epsilon$. Let $q\in \mathbb{N}$ be the number of nodes in an SPN and let $i'\in[q]$ be an index. For any $f\in \dist[s]_{\cF}$, we can prove that any sub tree of $f$ - with at least 1 sum node in the sub tree - rooted at node $i'$, $f_{i'}$, and corresponding sub tree of $\hat{f}$, $\hat{f}_{i'}$, satisfy the following: $$\|\hat{f}_{i'} - f_{i'} \|_{1} \leq 2\sum_{j^{'} \in N_{i'}}W_{j^{'}}^{i'}+2\hat{n}\epsilon/3n+2\hat{k}\epsilon/3k$$ where $\hat{n}$ is the dimensionality of the domain of the sub trees, $\hat{k}$ is the number of mixing weights in the sub trees, $N_{i'}$ is the subset of $[e]$ that represent the negligible leaves in $f_{i'}$ and $W_{j^{'}}^{i'}$ represents the product of the mixing weights along the unique path from (negligible) leaf $j'$ up to node $i'$. We can prove this via induction over the height of the sub trees.

\textbf{Base case:} Our base case consists of two separate possibilities. All sub trees (that contain at least 1 sum node) are built on top of either: 1) A sub tree of height 1 with a single sum node and some leaves or 2) a sub tree of height 2 rooted with a sum node. The sub trees of type 2) have roots that are connected to product nodes, which are further connected to leaves.

\emph{Case 1:} Let $f_{i'}$ represent a sub tree of height 1 rooted with a sum node. Let $m$ be the number of children of $f_i'$ and $\hat{f}_{i'}$. The children of the root nodes of $f_{i'}$ and $\hat{f_{i'}}$ are given by $f_{k'}$ and $\hat{f}_{k'}$ respectively, where $k'\in[m]$. Thus we have

\begin{align*}
   \Vert \hat{f}_{i'} - f_{i'} \Vert_{1} &= \Big\Vert\sum_{k'=1}^{m}(\hat{w}_{k'}\hat{f_{k'}}-w_{k'}f_{k'})\Big\Vert_{1} \leq \Big\Vert\sum_{k'=1}^{m} w_{k'}  (\hat{f}_{k'} - f_{k'})\Big\Vert_{1}  + \Big\Vert\sum_{k' =1}^{m}( \hat{w}_{k'} - w_{k'}) \hat{f}_{k'} \Big\Vert_{1}\\
    &\leq \sum_{k'=1}^{m} w_{k'} \Vert \hat{f}_{k'} - f_{k'} \Vert_{1} +  \sum_{k' =1}^{m} \vert \hat{w}_{k'} - w_{k'} \vert \Vert \hat{f}_{k'} \Vert_{1}\\
    &= \sum_{k'\in N_{i'}} w_{k'}\Vert \hat{f}_{k'} - f_{k'} \Vert_{1} + \sum_{k'\not\in N_{i'}} w_{k'} \Vert \hat{f}_{k'} - f_{k'} \Vert_{1}  \sum_{k' =1}^{m} \vert \hat{w}_{k'} - w_{k'} \vert \Vert \hat{f}_{k'} \Vert_{1}\\
    &\leq 2\sum_{k' \in N_{i'}} w_{k'}  +2\epsilon/3n+2\hat{k}\epsilon/3k \leq 2\sum_{k' \in N_{i'}} w_{k'}  +2\hat{n}\epsilon/3n+2\hat{k}\epsilon/3k \\
    &=  2\sum_{j' \in N_{i'}} W_{j'}^{i'}  +2\hat{n}\epsilon/3n+2\hat{k}\epsilon/3k
\end{align*}

\emph{Case 2:}
Let $f_{i'}$ represent a sub tree of height 1 rooted with a sum node. Let $m$ be the number of children of $f_i'$ and $\hat{f}_{i'}$. The children of the root nodes of $f_{i'}$ and $\hat{f_{i'}}$ are given by $f_{k'}$ and $\hat{f}_{k'}$ respectively, where $k'\in[m]$. We define $C_j$ as the set of nodes that are children of node $j$. Thus we have

\begin{align*}
   \Vert \hat{f}_{i'} - f_{i'} \Vert_{1} &= \Big\Vert\sum_{k'=1}^{m}(\hat{w}_{k'}\hat{f_{k'}}-w_{k'}f_{k'})\Big\Vert_{1} \leq \Big\Vert\sum_{k'=1}^{m} w_{k'}  (\hat{f}_{k'} - f_{k'})\Big\Vert_{1}  + \Big\Vert\sum_{k' =1}^{m}( \hat{w}_{k'} - w_{k'}) \hat{f}_{k'} \Big\Vert_{1}\\
    &\leq \sum_{k'=1}^{m} w_{k'} \Vert \hat{f}_{k'} - f_{k'} \Vert_{1} +  \sum_{k' =1}^{m} \vert \hat{w}_{k'} - w_{k'} \vert \Vert \hat{f}_{k'} \Vert_{1}\\
    &= \sum_{k' = 1}^{m} w_{k'} \Big\Vert \prod_{j \in C_{k'}}{\hat{f}_{j}} - \prod_{j \in C_{k'}}{f_{j}} \Big\Vert_{1} +  \sum_{k' =1} \vert \hat{w}_{k'} - w_{k'} \vert \Vert \hat{f}_{k'} \Vert_{1}\\
    &\leq \sum_{k' = 1}^{m} w_{k'} \left(
    \sum_{j \in C_{k'}\cap N_{k'}} 2 + \sum_{j \in C_{k'}\cap \bar{N}_{k'}}2\epsilon/3n \right) +   \sum_{k' =1}^{m} \vert \hat{w}_{k'} - w_{k'} \vert \Vert \hat{f}_{k'} \Vert_{1}\\
    &\leq 2\sum_{j' \in N_{i'}} w_{j'}  +2\hat{n}\epsilon/3n+2\hat{k}\epsilon/3k\\
    &= 2\sum_{j' \in N_{i'}} W_{j'}^{i'}+2\hat{n}\epsilon/3n  +2\hat{k}\epsilon/3k
\end{align*}

Where the first two inequalities follow from the properties of a norm, and the third inequality follows from proposition \ref{prop:TVprod}.

\textbf{Inductive Step:} Let $m$ be the number of children of the root nodes of SPNs $f_{i'}$ and $\hat{f}_{i'}$. For SPNs $f_{i'}$ and $\hat{f}_{i'}$, the children of their root node are given by $f_{k'}$ and $\hat{f}_{k'}$ respectively, where $k' \in [m]$ is a specific index. Assume the inductive hypothesis holds between each of the corresponding $m$ children. In the case that the root is a product node we have the following

\begin{align*}
 \Vert \hat{f}_{i'} - f_{i'} \Vert_{1} &= \Big\Vert \prod_{k' = 1}^{m}{\hat{f}_{k'}} - \prod_{k' = 1}^{m}{f_{k'}} \Big\Vert_{1} \leq \sum_{k'=1}^{m}\Vert \hat{f}_{k'} - f_{k'} \Vert_{1} \\
    &\leq \sum_{k'=1}^{m}\bigg( 2\sum_{ l' \in N_{k'}}W_{l'}^{k'}+(2\hat{k}_{k'}\cdot\epsilon)/3k+(2\hat{n}_{k'}\cdot\epsilon)/3n\bigg)\\
    &= 2\sum_{j' \in N_{i'}}W_{j'}^{i'}+2\hat{k}\epsilon/3k+2\hat{n}\epsilon/3n
\end{align*}

Where the first inequality holds by Proposition~\ref{prop:TVprod}. In the case that the root node is a sum node, we have the following
\begin{align*}
   \Vert \hat{f}_{i'} - f_{i'} \Vert_{1} &= \Big\Vert\sum_{k'=1}^{m}(\hat{w}_{k'}\hat{f_{k'}}-w_{k'}f_{k'})\Big\Vert_{1} \leq \Big\Vert\sum_{k'=1}^{m} w_{k'}  (\hat{f}_{k'} - f_{k'})\Big\Vert_{1}  + \Big\Vert\sum_{k' =1}^{m}( \hat{w}_{k'} - w_{k'}) \hat{f}_{k'} \Big\Vert_{1}\\
    &\leq \sum_{k'=1}^{m} w_{k'} \Vert \hat{f}_{k'} - f_{k'} \Vert_{1} +  \sum_{k' =1}^{m} \vert \hat{w}_{k'} - w_{k'} \vert \Vert \hat{f}_{k'} \Vert_{1}\\
    &\leq \sum_{k'=1}^{m} w_{k'} \cdot \left( 2\sum_{l' \in N_{k'}}W_{l'}^{k'}+(2\hat{k}_{k'}\cdot\epsilon)/3k+2\hat{n}\epsilon/3n\right)\\
    &\quad+ \sum_{k'=1}^{m} 2\epsilon/3k \cdot 1\\
    &\leq 2\sum_{j' \in N_{j'}}W_{j'}^{i'}+2\hat{k}\epsilon/3k+2\hat{n}\epsilon/3n
\end{align*}

Which proves the inductive hypothesis. Let $o$ be the root node. Clearly, for any negligible leaf $j'$, $W_{j'}^{o}$ is the path weight of the leaf. For the SPNs $f$ and $\hat{f}$ we have 

\begin{align*}
   \Vert \hat{f} - f \Vert_{1} &\leq 2\sum_{j' \in N_o}W_{j'}^{o}+2k\epsilon/3k+2n\epsilon/3n\\
   &\leq 2\sum_{j' \in N_o}\epsilon/3e + 2\epsilon/3+2\epsilon/3\\
   &\leq 2\epsilon/3 + 2\epsilon/3 + 2\epsilon/3 \\
   &= 2\epsilon
\end{align*}

Since $\TV(\hat{f},f) = \frac{1}{2}\|\hat{f}-f\|_{1}$, this completes the proof.

\end{proof}

\end{document}